\documentclass[letterpaper]{article}
\usepackage[latin9]{inputenc}
\usepackage{float}
\usepackage{url}
\usepackage{amsthm}
\usepackage{amsmath}
\usepackage{amssymb}
\usepackage{graphicx}
\usepackage{esint}
\usepackage[authoryear]{natbib}

\makeatletter

\floatstyle{ruled}
\newfloat{algorithm}{tbp}{loa}
\providecommand{\algorithmname}{Algorithm}
\floatname{algorithm}{\protect\algorithmname}

\theoremstyle{plain}
\newtheorem{thm}{\protect\theoremname}
\theoremstyle{plain}
\newtheorem{prop}[thm]{\protect\propositionname}
\ifx\proof\undefined
\newenvironment{proof}[1][\protect\proofname]{\par
\normalfont\topsep6\p@\@plus6\p@\relax
\trivlist
\itemindent\parindent
\item[\hskip\labelsep
\scshape
#1]\ignorespaces
}{%
\endtrivlist\@endpefalse
}
\providecommand{\proofname}{Proof}
\fi

\usepackage{times}
\usepackage{graphicx} 
\usepackage{subfigure} 

\usepackage{natbib}

\usepackage{algorithm}
\usepackage{algorithmic}

\usepackage{hyperref}

\usepackage[accepted]{icml2014}

\icmltitlerunning{Kernel Adaptive Metropolis Hastings}

\makeatother

\providecommand{\propositionname}{Proposition}
\providecommand{\theoremname}{Theorem}

\begin{document}
\twocolumn[
\icmltitle{Kernel Adaptive Metropolis-Hastings}
\icmlauthor{Dino Sejdinovic${}^\star$}{dino@gatsby.ucl.ac.uk}
\icmlauthor{Heiko Strathmann${}^\star$}{ucabhst@gatsby.ucl.ac.uk}
\icmlauthor{Maria Lomeli Garcia${}^\star$}{mlomeli@gatsby.ucl.ac.uk}
\icmlauthor{Christophe Andrieu${}^\ddagger$}{c.andrieu@bristol.ac.uk}
\icmlauthor{Arthur Gretton${}^\star$}{arthur.gretton@gmail.com}
\icmladdress{${}^\star$Gatsby Unit, CSML, University College London, UK and ${}^\ddagger$School of Mathematics, University of Bristol, UK}
\icmlkeywords{kernel methods, adaptive MCMC, pseudo-marginal MCMC}
\vskip 0.2in
]

\global\long\def\Mz{M_{\mathbf{z},y}}
\global\long\def\muz{\mu_{\mathbf{z}}}
\global\long\def\aj{\alpha^{(j)}}
\global\long\def\ajt{\left(\alpha^{(j)}\right)^{\top}}
\vspace{-0.2cm} 
\begin{abstract}
A Kernel Adaptive Metropolis-Hastings algorithm is introduced, for
the purpose of sampling from a target distribution with strongly nonlinear
support. The algorithm embeds the trajectory of the Markov chain into
a reproducing kernel Hilbert space (RKHS), such that the feature space
covariance of the samples informs the choice of proposal. The procedure
is computationally efficient and straightforward to implement, since
the RKHS moves can be integrated out analytically: our proposal distribution
in the original space is a normal distribution whose mean and covariance
depend on where the current sample lies in the support of the target
distribution, and adapts to its local covariance structure. Furthermore,
the procedure requires neither gradients nor any other higher order
information about the target, making it particularly attractive for
contexts such as Pseudo-Marginal MCMC. Kernel Adaptive Metropolis-Hastings
outperforms competing fixed and adaptive samplers on multivariate,
highly nonlinear target distributions, arising in both real-world
and synthetic examples.\vspace{-0.2cm} 
\end{abstract}

\section{Introduction}

The choice of the proposal distribution is known to be crucial for
the design of Metropolis-Hastings algorithms, and methods for adapting
the proposal distribution to increase the sampler's efficiency based
on the history of the Markov chain have been widely studied. These
methods often aim to learn the covariance structure of the target
distribution, and adapt the proposal accordingly. Adaptive MCMC samplers
were first studied by \citet{Haario1999,Haario2001}, where the authors
propose to update the proposal distribution along the sampling process.
Based on the chain history, they estimate the covariance of the target
distribution and construct a Gaussian proposal centered at the current
chain state, with a particular choice of the scaling factor from \citet{Gelman96}.
More sophisticated schemes are presented by \citet{Andrieu08}, e.g.,
adaptive scaling, component-wise scaling, and principal component
updates.

While these strategies are beneficial for distributions that show
high anisotropy (e.g., by ensuring the proposal uses the right scaling
in all principal directions), they may still suffer from low acceptance
probability and slow mixing when the target distributions are strongly
nonlinear, and the directions of large variance depend on the current
location of the sampler in the support. In the present work, we develop
an adaptive Metropolis-Hastings algorithm in which samples are mapped
to a reproducing kernel Hilbert space, and the proposal distribution
is chosen according to the covariance in this feature space \citep{SchSmoMul98,SmoMikSchWil01}.
Unlike earlier adaptive approaches, the resulting proposal distributions
are locally adaptive in input space, and oriented towards nearby regions
of high density, rather than simply matching the global covariance
structure of the distribution. Our approach combines a move in the
feature space with a stochastic step towards the nearest input space
point, where the feature space move can be analytically integrated
out. Thus, the implementation of the procedure is straightforward:
the proposal is simply a multivariate Gaussian in the input space,
with location-dependent covariance which is informed by the feature
space representation of the target. Furthermore, the resulting Metropolis-Hastings
sampler only requires the ability to evaluate the unnormalized density
of the target \citep[or its unbiased estimate, as in Pseudo-Marginal MCMC of][]{andrieu2009pseudo},
and no gradient evaluation is needed, making it applicable to situations
where more sophisticated schemes based on Hamiltonian Monte Carlo
(HMC) or Metropolis Adjusted Langevin Algorithms (MALA) \citep{Roberts2003,RSSB:RSSB765}
cannot be applied.

We begin our presentation in Section \ref{sec:Background}, with a
brief overview of existing adaptive Metropolis approaches; we also
review covariance operators in the RKHS. Based on these operators,
we describe a sampling strategy for Gaussian measures in the RKHS
in Section \ref{sec:Sampling-in-RKHS}, and introduce a cost function
for constructing proposal distributions. In Section \ref{sec:Kernel-adaptive-proposal},
we outline our main algorithm, termed Kernel Adaptive Metropolis-Hastings
(MCMC Kameleon). We provide experimental comparisons with other fixed
and adaptive samplers in Section \ref{sec:Experiments}, where we
show superior performance in the context of Pseudo-Marginal MCMC for
Bayesian classification, and on synthetic target distributions with
highly nonlinear shape.

\section{\label{sec:Background}Background}

\paragraph{Adaptive Metropolis Algorithms.}

Let $\mathcal{X}=\mathbb{R}^{d}$ be the domain of interest, and denote
the unnormalized target density on $\mathcal{X}$ by $\pi$. Additionally,
let $\Sigma_{t}=\Sigma_{t}(x_{0},x_{1},\ldots,x_{t-1})$ denote an
estimate of the covariance matrix of the target density based on the
chain history $\left\{ x_{i}\right\} _{i=0}^{t-1}$. The original
adaptive Metropolis at the current state of the chain state $x_{t}=y$
uses the proposal
\begin{equation}
q_{t}(\cdot|y)=\mathcal{N}(y,\nu^{2}\Sigma_{t}),\label{eq: am_proposal}
\end{equation}
where $\nu=2.38/\sqrt{d}$ is a fixed scaling factor from \citet{Gelman96}.
This choice of scaling factor was shown to be optimal (in terms of
efficiency measures) for the usual Metropolis algorithm. While this
optimality result does not hold for Adaptive Metropolis, it can nevertheless
be used as a heuristic. Alternatively, the scale $\nu$ can also be
adapted at each step as in \citet[Algorithm 4]{Andrieu08} to obtain
the acceptance rate from \citet{Gelman96}, $a^{*}=0.234$.

\paragraph{RKHS Embeddings and Covariance Operators.}

According to the Moore-Aronszajn theorem \citep[p. 19]{BerTho04},
for every symmetric, positive definite function (\emph{kernel}) $k:\mathcal{X}\times\mathcal{X}\to\mathbb{R}$,
there is an associated reproducing kernel Hilbert space $\mathcal{H}_{k}$
of real-valued functions on $\mathcal{X}$ with reproducing kernel
$k$. The map $\varphi:\mathcal{X}\to\mathcal{H}_{k}$, $\varphi:x\mapsto k(\cdot,x)$
is called the canonical feature map of $k$. This feature map or embedding
of a single point can be extended to that of a probability measure
$P$ on $\mathcal{X}$: its kernel embedding is an element $\mu_{P}\in\mathcal{H}_{k}$,
given by $\mu_{P}=\int k(\cdot,x)\, dP(x)$ \citep{BerTho04,FukBacJor04,SmoGreSonSch07}.
If a measurable kernel $k$ is bounded, it is straightforward to show
using the Riesz representation theorem that the mean embedding $\mu_{k}(P)$
exists for all probability measures on $\mathcal{X}$. For many interesting
bounded kernels $k$, including the Gaussian, Laplacian and inverse
multi-quadratics, the kernel embedding $P\mapsto\mu_{P}$ is injective.
Such kernels are said to be \emph{characteristic} \citep{SriGreFukLanetal10,Sriperumbudur2011},
since each distribution is uniquely characterized by its embedding
(in the same way that every probability distribution has a unique
characteristic function).  The kernel embedding $\mu_{P}$ is the
representer of expectations of smooth functions w.r.t. $P$, i.e.,
$\forall f\in\mathcal{H}_{k}$, $\left\langle f,\mu_{P}\right\rangle _{\mathcal{H}_{k}}=\int f(x)dP(x)$.
Given samples $\mathbf{z}=\left\{ z_{i}\right\} _{i=1}^{n}\sim P$,
the embedding of the empirical measure is $\muz=\frac{1}{n}\sum_{i=1}^{n}k(\cdot,z_{i})$.

Next, the covariance operator $C_{P}:\mathcal{H}_{k}\to\mathcal{H}_{k}$
for a probability measure $P$ is given by $C_{P}=\int k(\cdot,x)\otimes k(\cdot,x)\, dP(x)-\mu_{P}\otimes\mu_{P}$
\citep{Baker73,FukBacJor04}, where for $a,b,c\in\mathcal{H}_{k}$
the tensor product is defined as $(a\otimes b)c=\left\langle b,c\right\rangle _{\mathcal{H}_{k}}a$.
The covariance operator has the property that $\forall f,g\in\mathcal{H}_{k}$,
$\left\langle f,C_{P}g\right\rangle _{\mathcal{H}_{k}}=\mathbb{E}_{P}(fg)-\mathbb{E}_{P}f\mathbb{E}_{P}g$.

Our approach is based on the idea that the nonlinear support of a
target density may be learned using Kernel Principal Component Analysis
(Kernel PCA) \citep{SchSmoMul98,SmoMikSchWil01}, this being linear
PCA on the empirical covariance operator in the RKHS, $C_{\mathbf{z}}=\frac{1}{n}\sum_{i=1}^{n}k(\cdot,z_{i})\otimes k(\cdot,z_{i})-\muz\otimes\muz$,
computed on the sample $\mathbf{z}$ defined above. The empirical
covariance operator behaves as expected: applying the tensor product
definition gives $\left\langle f,C_{\mathbf{z}}g\right\rangle _{\mathcal{H}_{k}}=\frac{1}{n}\sum_{i=1}^{n}f(z_{i})g(z_{i})-\left(\frac{1}{n}\sum_{i=1}^{n}f(z_{i})\right)\left(\frac{1}{n}\sum_{i=1}^{n}g(z_{i})\right)$.
By analogy with algorithms which use linear PCA directions to inform
M-H proposals \citep[Algorithm 8]{Andrieu08}, nonlinear PCA directions
can be encoded in the proposal construction, as described in Appendix
\ref{sec:Principal-components-proposals}. Alternatively, one can
focus on a Gaussian measure on the RKHS determined by the empirical
covariance operator $C_{\mathbf{z}}$ rather than extracting its eigendirections,
which is the approach we pursue in this contribution. This generalizes
the proposal \eqref{eq: am_proposal}, which considers the Gaussian
measure induced by the empirical covariance matrix on the original
space.

\section{\label{sec:Sampling-in-RKHS}Sampling in RKHS}

We next describe the proposal distribution at iteration $t$ of the
MCMC chain. We will assume that a subset of the chain history, denoted
$\mathbf{z}=\left\{ z_{i}\right\} _{i=1}^{n}$, $n\leq t-1$, is
available. Our proposal is constructed by first considering the samples
in the RKHS associated to the empirical covariance operator, and then
performing a gradient descent step on a cost function associated with
those samples.

\paragraph{Gaussian Measure of the Covariance Operator.}

We will work with the Gaussian measure on the RKHS $\mathcal{H}_{k}$
with mean $k(\cdot,y)$ and covariance $\nu^{2}C_{\mathbf{z}}$, where
$\mathbf{z}=\left\{ z_{i}\right\} _{i=1}^{n}$ is the subset of the
chain history. While there is no analogue of a Lebesgue measure in
an infinite dimensional RKHS, it is instructive (albeit with some
abuse of notation) to denote this measure in the ``density form''
$\mathcal{N}(f\,;\, k(\cdot,y),\nu^{2}C_{\mathbf{z}})\propto\exp\left(-\frac{1}{2\nu^{2}}\left\langle f-k(\cdot,y),C_{\mathbf{z}}^{-1}(f-k(\cdot,y))\right\rangle _{\mathcal{H}_{k}}\right)$.
As $C{}_{\mathbf{z}}$ is a finite-rank operator, this measure is
supported only on a finite-dimensional affine space $k(\cdot,y)+\mathcal{H}_{\mathbf{z}}$,
where $\mathcal{H}_{\mathbf{z}}=\textrm{span}\left\{ k(\cdot,z_{i})\right\} _{i=1}^{n}$
is the subspace spanned by the canonical features of $\mathbf{z}$.
It can be shown that a sample from this measure has the form $f=k(\cdot,y)+\sum_{i=1}^{n}\beta_{i}\left[k(\cdot,z_{i})-\muz\right],$
where $\beta\sim\mathcal{N}(0,\frac{\nu^{2}}{n}I)$ is isotropic.
Indeed, to see that $f$ has the correct covariance structure, note
that:\vspace{-0.2cm}

\begin{align*}
 & \mathbb{E}\left[\left(f-k(\cdot,y)\right)\otimes\left(f-k(\cdot,y)\right)\right]\\
 & =\mathbb{E}\left[\sum_{i=1}^{n}\sum_{j=1}^{n}\beta_{i}\beta_{j}\left(k(\cdot,z_{i})-\mu_{\mathbf{z}}\right)\otimes\left(k(\cdot,z_{j})-\mu_{\mathbf{z}}\right)\right]\\
 & =\frac{\nu^{2}}{n}\sum_{i=1}^{n}\left(k(\cdot,z_{i})-\mu_{\mathbf{z}}\right)\otimes\left(k(\cdot,z_{i})-\mu_{\mathbf{z}}\right)=\nu^{2}C_{\mathbf{z}}.
\end{align*}

{\small }Due to the equivalence in the RKHS between a Gaussian
measure and a Gaussian Process (GP) \citep[Ch. 4]{BerTho04}, we can
think of the RKHS samples $f$ as trajectories of the GP with mean
$m(x)=k(x,y)$ and covariance function\vspace{-0.2cm}
\begin{align*}
\kappa(x,x') & =\textrm{cov}\left[f(x),f(x')\right]\\
 & =\frac{\nu^{2}}{n}\sum_{i=1}^{n}\left(k(x,z_{i})-\mu_{\mathbf{z}}(x)\right)\left(k(x',z_{i})-\mu_{\mathbf{z}}(x')\right).
\end{align*}
The covariance function $\kappa$ of this GP is therefore the kernel
$k$ convolved with itself with respect to the empirical measure associated
to the samples $\mathbf{z}$, and draws from this GP therefore lie
in a smaller RKHS; see \citet[p. 21]{Saitoh1997} for details.

\paragraph{Obtaining Target Samples through Gradient Descent.}

We have seen how to obtain the RKHS sample $f=k(\cdot,y)+\sum_{i=1}^{n}\beta_{i}\left[k(\cdot,z_{i})-\muz\right]$
from the Gaussian measure in the RKHS. This sample does not in general
have a corresponding pre-image in the original domain $\mathcal{X}=\mathbb{R}^{d}$;
i.e., there is no point $x_{*}\in\mathcal{X}$ such that $f=k(\cdot,x_{*})$.
If there were such a point, then we could use it as a proposal in
the original domain. Therefore, we are ideally looking for a point
$x^{*}\in\mathcal{X}$ whose canonical feature map $k(\cdot,x^{*})$
is close to $f$ in the RKHS norm. We consider the optimization problem\vspace{-0.5cm}

{\footnotesize 
\begin{multline*}
\arg\min_{x\in\mathcal{X}}\left\Vert k\left(\cdot,x\right)-f\right\Vert _{\mathcal{H}_{k}}^{2}=\\
\quad\arg\min_{x\in\mathcal{X}}\left\{ k(x,x)-2k(x,y)-2\sum_{i=1}^{n}\beta_{i}\left[k(x,z_{i})-\muz(x)\right]\right\} .
\end{multline*}
}{\small }In general, this is a non-convex minimization problem,
and may be difficult to solve \citep{BakWesSch03}. Rather than solving
it for every new vector of coefficients $\beta$, which would lead
to an excessive computational burden for every proposal made, we simply
make a single descent step along the gradient of the cost function,\vspace{-0.2cm}
\begin{equation}
g(x)=k(x,x)-2k(x,y)-2\sum_{i=1}^{n}\beta_{i}\left[k(x,z_{i})-\muz(x)\right],\label{eq: g_function}
\end{equation}
i.e., the proposed new point is
\[
x^{*}=y-\eta\nabla_{x}g(x)|_{x=y}+\xi,
\]
where $\eta$ is a gradient step size parameter and $\xi\sim\mathcal{N}(0,\gamma^{2}I)$
is an additional isotropic 'exploration' term \emph{after} the gradient
step. It will be useful to split the scaled gradient at $y$ into
two terms as $\eta\nabla_{x}g(x)|_{x=y}=\eta\left(a_{y}-M_{\mathbf{z},y}H\beta\right)$,
where $a_{y}=\nabla_{x}k(x,x)|_{x=y}-2\nabla_{x}k(x,y)|_{x=y}$, \vspace{-0.1cm}
\begin{equation}
\Mz=2\left[\nabla_{x}k(x,z_{1})|_{x=y},\ldots,\nabla_{x}k(x,z_{n})|_{x=y}\right]\label{eq: Mzy matrix}
\end{equation}
 is a $d\times n$ matrix, and $H=I-\frac{1}{n}\mathbf{1}_{n\times n}$
is the $n\times n$ centering matrix.

Figure \ref{fig: gPlots} plots $g(x)$ and its gradients for several
samples of $\beta$-coefficients, in the case where the underlying
$\mathbf{z}$-samples are from the two-dimensional nonlinear Banana
target distribution of \citet{Haario1999}. It can be seen that $g$
may have multiple local minima, and that it varies most along the
high-density regions of the Banana distribution.

\begin{figure}
\begin{centering}
\includegraphics[bb=15bp 15bp 201bp 104bp,clip]{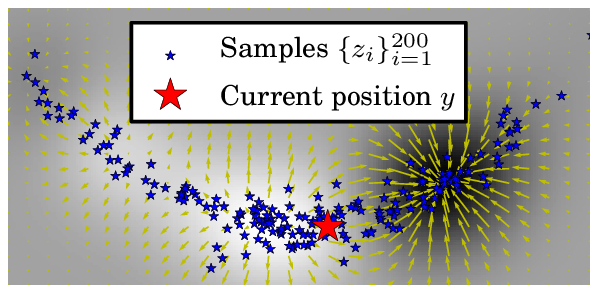}
\par\end{centering}

\begin{centering}
\includegraphics[bb=15bp 15bp 201bp 104bp,clip]{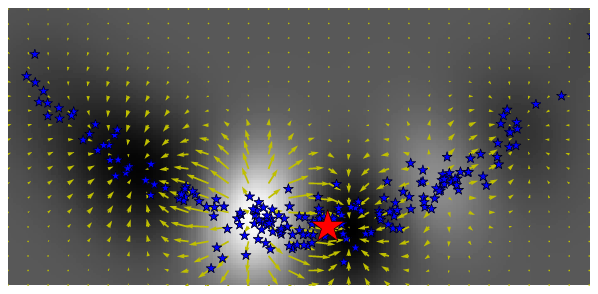}
\par\end{centering}

\caption{\label{fig: gPlots}Heatmaps (white denotes large) and gradients of
$g(x)$ for two samples of $\beta$ and fixed $\mathbf{z}$.}

\end{figure}

\section{\label{sec:Kernel-adaptive-proposal}MCMC Kameleon Algorithm}

\begin{algorithm}
\textbf{MCMC Kameleon}

\emph{Input}: unnormalized target $\pi$, subsample size $n$, scaling
parameters $\nu,\gamma,$ adaptation probabilities $\left\{ p_{t}\right\} _{t=0}^{\infty}$,
kernel $k$, 
\begin{itemize}
\item At iteration $t+1$,

\begin{enumerate}
\item With probability $p_{t}$, update a random subsample $\mathbf{z}=\left\{ z_{i}\right\} _{i=1}^{\min(n,t)}$
of the chain history $\left\{ x_{i}\right\} _{i=0}^{t-1}$,
\item Sample proposed point $x^{*}$ from $q_{\mathbf{z}}(\cdot|x_{t})=\mathcal{N}(x_{t},\gamma^{2}I+\nu^{2}M_{\mathbf{z},x_{t}}HM_{\mathbf{z},x_{t}}^{\top})$,
where $M_{{\bf z},x_{t}}$is given in Eq. \eqref{eq: Mzy matrix}
and $H=I-\frac{1}{n}\mathbf{1}_{n\times n}$ is the centering matrix, 
\item Accept/Reject with the Metropolis-Hastings acceptance probability
$A(x_{t},x^{*})$ in Eq. \eqref{eq:MetAcceptProb},\vspace{-0.5cm}

\begin{eqnarray*}
x_{t+1} & = & \begin{cases}
x^{*}, & \textrm{w.p.}\; A(x_{t},x^{*}),\\
x_{t}, & \textrm{w.p.}\;1-A(x_{t},x^{*}).
\end{cases}
\end{eqnarray*}
\end{enumerate}
\end{itemize}
\end{algorithm}

\subsection{Proposal Distribution}

We now have a recipe to construct a proposal that is able to adapt
to the local covariance structure for the current chain state $y$.
This proposal depends on a subset of the chain history $\mathbf{z}$,
and is denoted by $q_{\mathbf{z}}(\cdot|y)$. While we will later
simplify this proposal by integrating out the moves in the RKHS, it
is instructive to think of the proposal generating process as:
\begin{enumerate}
\item Sample $\beta\sim\mathcal{N}(0,\nu^{2}I)$ ($n\times1$ normal of
RKHS coefficients). 

\begin{itemize}
\item This represents an RKHS sample $f=k(\cdot,y)+\sum_{i=1}^{n}\beta_{i}\left[k(\cdot,z_{i})-\muz\right]$
which is the goal of the cost function $g(x)$.
\end{itemize}
\item Move along the gradient of $g$: $x^{*}=y-\eta\nabla_{x}g(x)|_{x=y}+\xi.$ 

\begin{itemize}
\item This gives a proposal $x^{*}|y,\beta\sim\mathcal{N}(y-\eta a_{y}+\eta M_{\mathbf{z},y}H\beta,\gamma^{2}I)$
($d\times1$ normal in the original space).
\end{itemize}
\end{enumerate}
Our first step in the derivation of the explicit proposal density
is to show that as long as $k$ is a differentiable positive definite
kernel, the term $a_{y}$ vanishes. 
\begin{prop}
Let $k$ be a differentiable positive definite kernel. Then \textup{$a_{y}=\nabla_{x}k(x,x)|_{x=y}-2\nabla_{x}k(x,y)|_{x=y}=0$.}
\end{prop}
Since $a_{y}=0$, the gradient step size $\eta$ always appears together
with $\beta$, so we merge $\eta$ and the scale $\nu$ of the $\beta$-coefficients
into a single scale parameter, and set $\eta=1$ henceforth. Furthermore,
since both $p(\beta)$ and $p_{\mathbf{z}}(x^{*}|y,\beta)$ are multivariate
Gaussian densities, the proposal density $q_{\mathbf{z}}(x^{*}|y)=\int p(\beta)p_{\mathbf{z}}(x^{*}|y,\beta)d\beta$
can be computed analytically. We therefore get the following closed
form expression for the proposal distribution.
\begin{prop}
\textup{$q_{\mathbf{z}}(\cdot|y)=\mathcal{N}(y,\gamma^{2}I+\nu^{2}\Mz H\Mz^{\top})$.}
\end{prop}
Proofs of the above Propositions are given in Appendix \ref{sec:Proofs}. 

\begin{figure}
\begin{centering}
\includegraphics{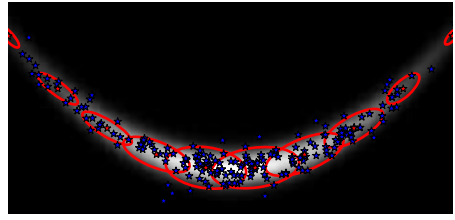}\vspace{-0.5cm}
\par\end{centering}

\caption{\label{fig: proposal_contours}$95\%$ contours (red) of proposal
distributions evaluated at a number of points, for the first two dimensions
of the banana target of \citet{Haario1999}. Underneath is the density
heatmap, and the samples (blue) used to construct the proposals.}
 
\end{figure}

With the derived proposal distribution, we proceed with the standard
Metropolis-Hastings accept/reject scheme, where the proposed sample
$x^{*}$ is accepted with probability\vspace{-0.2cm} 

\begin{eqnarray}
A(x_{t},x^{*}) & = & \min\left\{ 1,\frac{\pi(x^{*})q_{\mathbf{z}}(x_{t}|x^{*})}{\pi(x_{t})q_{\mathbf{z}}(x^{*}|x_{t})}\right\} ,\label{eq:MetAcceptProb}
\end{eqnarray}
giving rise to the MCMC Kameleon Algorithm. Note that each $\pi(x^{*})$
and $\pi(x_{t})$ could be replaced by their unbiased estimates without
impacting the invariant distribution \citep{andrieu2009pseudo}.\vspace{-0.2cm}

The constructed family of proposals encodes local structure of the
target distribution, which is learned based on the subsample $\mathbf{z}$.
Figure \ref{fig: proposal_contours} depicts the regions that contain
95\% of the mass of the proposal distribution $q_{\mathbf{z}}(\cdot|y)$
at various states $y$ for a fixed subsample $\mathbf{z}$, where
the Banana target is used (details in Section \ref{sec:Experiments}).
More examples of proposal contours can be found in Appendix \ref{sec:Proposal-Contours-for}.

\subsection{Properties of the Algorithm}

\paragraph{The update schedule and convergence.}

MCMC Kameleon requires a subsample $\mathbf{z}=\left\{ z_{i}\right\} _{i=1}^{n}$
at each iteration of the algorithm, and the proposal distribution
$q_{\mathbf{z}}(\cdot|y)$ is updated each time a new subsample $\mathbf{z}$
is obtained. It is well known that a chain which keeps adapting the
proposal distribution need not converge to the correct target \citep{Andrieu08}.
To guarantee convergence, we introduce adaptation probabilities $\left\{ p_{t}\right\} _{t=0}^{\infty}$,
such that $p_{t}\to0$ and $\sum_{t=1}^{\infty}p_{t}=\infty$, and
at iteration $t$ we update the subsample $\mathbf{z}$ with probability
$p_{t}$. As adaptations occur with decreasing probability, Theorem
1 of \citet{RobertsRosenthal2007} implies that the resulting algorithm
is ergodic and converges to the correct target. Another straightforward
way to guarantee convergence is to fix the set $\mathbf{z}=\left\{ z_{i}\right\} _{i=1}^{n}$
after a ``burn-in'' phase; i.e., to stop adapting \citet[Proposition 2]{RobertsRosenthal2007}.
In this case, a ``burn-in'' phase is used to get a rough sketch
of the shape of the distribution: the initial samples need not come
from a converged or even valid MCMC chain, and it suffices to have
a scheme with good exploratory properties, e.g., \citet{welling2011bayesian}.
In MCMC Kameleon, the term $\gamma$ allows exploration in the initial
iterations of the chain (while the subsample $\mathbf{z}$ is still
not informative about the structure of the target) and provides regularization
of the proposal covariance in cases where it might become ill-conditioned.
Intuitively, a good approach to setting $\gamma$ is to slowly decrease
it with each adaptation, such that the learned covariance progressively
dominates the proposal.\vspace{-0.2cm}

\paragraph{Symmetry of the proposal.}

In \citet{Haario2001}, the proposal distribution is asymptotically
symmetric due to the vanishing adaptation property. Therefore, the
authors compute the standard Metropolis acceptance probability. In
our case, the proposal distribution is a Gaussian with mean at the
current state of the chain $x_{t}=y$ and covariance $\gamma^{2}I+\nu^{2}\Mz H\Mz^{\top}$,
where $\Mz$ depends both on the current state $y$ and a random subsample
$\mathbf{z}=\left\{ z_{i}\right\} _{i=1}^{n}$ of the chain history
$\left\{ x_{i}\right\} _{i=0}^{t-1}$. This proposal distribution
is never symmetric (as covariance of the proposal always depends on
the current state of the chain), and therefore we use the Metropolis-Hastings
acceptance probability to reflect this.\vspace{-0.2cm}

\paragraph{Relationship to MALA and Manifold MALA.}

The Metropolis Adjusted Langevin Algorithm (MALA) algorithm uses information
about the gradient of the log-target density at the current chain
state to construct a proposed point for the Metropolis step. Our approach
does not require that the log-target density gradient be available
or computable. Kernel gradients in the matrix $\Mz$ are easily obtained
for commonly used kernels, including the Gaussian kernel (see section
\ref{sec:Examples-of-covariance}), for which the computational complexity
is equal to evaluating the kernel itself. Moreover, while standard
MALA simply shifts the mean of the proposal distribution along the
gradient and then adds an isotropic exploration term, our proposal
is centered at the current state, and it is the covariance structure
of the proposal distribution that coerces the proposed points to belong
to the high-density regions of the target. It would be straightforward
to modify our approach to include a drift term along the gradient
of the log-density, should such information be available, but it is
unclear whether this would provide additional performance gains. Further
work is required to elucidate possible connections between our approach
and the use of a preconditioning matrix \citep{Roberts2003} in the
MALA proposal; i.e., where the exploration term is scaled with appropriate
metric tensor information, as in Riemannian manifold MALA \citep{RSSB:RSSB765}.\vspace{-0.2cm}

\subsection{\label{sec:Examples-of-covariance}Examples of Covariance Structure
for Standard Kernels}

The proposal distributions in MCMC Kameleon are dependant on the choice
of the kernel $k$. To gain intuition regarding their covariance structure,
we give two examples below.

\paragraph{Linear kernel. }

In the case of a linear kernel $k(x,x')=x^{\top}x'$, we obtain $\Mz=2\left[\nabla_{x}x^{\top}z_{1}|_{x=y},\ldots,\nabla_{x}x^{\top}z_{n}|_{x=y}\right]=2\mathbf{Z}^{\top}$,
so the proposal is given by $q_{\mathbf{z}}(\cdot|y)=\mathcal{N}(y,\gamma^{2}I+4\nu^{2}\mathbf{Z}^{\top}H\mathbf{Z})$;
thus, the proposal simply uses the scaled empirical covariance $\mathbf{Z}^{\top}H\mathbf{Z}$
just like standard Adaptive Metropolis \citep{Haario1999}, with an
additional isotropic exploration component, and depends on $y$ only
through the mean.

\paragraph{Gaussian kernel. }

In the case of a Gaussian kernel $k(x,x')=\exp\left(-\frac{\left\Vert x-x'\right\Vert _{2}^{2}}{2\sigma^{2}}\right)$,
since $\nabla_{x}k(x,x')=\frac{1}{\sigma^{2}}k(x,x')(x'-x)$, we obtain\vspace{-0.2cm}
\[
\Mz=\frac{2}{\sigma^{2}}\left[k(y,z_{1})(z_{1}-y),\ldots,k(y,z_{n})(z_{n}-y)\right].
\]
Consider how this encodes the covariance structure of the target distribution:\vspace{-0.1cm}
{\small 
\begin{eqnarray}
R_{ij} & = & \gamma^{2}\delta_{ij}\nonumber \\
 & + & \frac{4\nu^{2}(n-1)}{\sigma^{4}n}\sum_{a=1}^{n}\left[k(y,z_{a})\right]^{2}(z_{a,i}-y_{i})(z_{a,j}-y_{j})\nonumber \\
 & - & \frac{4\nu^{2}}{\sigma^{4}n}\sum_{a\neq b}k(y,z_{a})k(y,z_{b})(z_{a,i}-y_{i})(z_{b,j}-y_{j}).\label{eq: Rgaussian}
\end{eqnarray}
}As the first two terms dominate, the previous points $z_{a}$ which
are close to the current state $y$ (for which $k(y,z_{a})$ is large)
have larger weights, and thus they have more influence in determining
the covariance of the proposal at $y$.

\paragraph{Mat\'{e}rn kernel.}

In the Mat\'{e}rn family of kernels $k_{\vartheta,\rho}(x,x')=\frac{2^{1-\vartheta}}{\Gamma(\vartheta)}\left(\frac{\left\Vert x-x'\right\Vert _{2}}{\rho}\right)^{\vartheta}K_{\vartheta}\left(\frac{\left\Vert x-x'\right\Vert _{2}}{\rho}\right)$,
where $K_{\vartheta}$ is the modified Bessel function of the second
kind, we obtain a form of the covariance structure very similar to
that of the Gaussain kernel. In this case, $\nabla_{x}k_{\vartheta,\rho}(x,x')=\frac{1}{2\rho^{2}(\vartheta-1)}k_{\vartheta-1,\rho}(x,x')(x'-x)$,
so the only difference (apart from the scalings) to \eqref{eq: Rgaussian}
is that the weights are now determined by a ``rougher'' kernel $k_{\vartheta-1,\rho}$
of the same family. \vspace{-0.2cm} 

\begin{figure}
\begin{centering}
\includegraphics[width=2in]{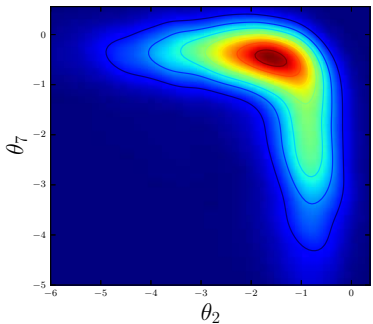}\vspace{-0.5cm}
\par\end{centering}

\caption{\label{fig:posterior2vs7}Dimensions 2 and 7 of the marginal hyperparameter
posterior on the UCI Glass dataset}
\end{figure}

\section{\label{sec:Experiments}Experiments}

In the experiments, we compare the following samplers: \textbf{(SM)}
Standard Metropolis with the isotropic proposal $q(\cdot|y)=\mathcal{N}(y,\nu^{2}I)$
and scaling $\nu=2.38/\sqrt{d}$, \textbf{(AM-FS)} Adaptive Metropolis
with a learned covariance matrix and fixed scaling $\nu=2.38/\sqrt{d}$,
\textbf{(AM-LS)} Adaptive Metropolis with a learned covariance matrix
and scaling learned to bring the acceptance rate close to $\alpha^{*}=0.234$
as described in \citet[Algorithm 4]{Andrieu08}, and \textbf{(KAMH-LS)}
MCMC Kameleon with the scaling $\nu$ learned in the same fashion
($\gamma$ was fixed to 0.2), and which also stops adapting the proposal
after the burn-in of the chain (in all experiments, we use a random
subsample \textbf{$\mathbf{z}$} of size $n=1000$, and a Gaussian
kernel with bandwidth selected according to the median heuristic).
We consider the following nonlinear targets: (1) the posterior distribution
of Gaussian Process (GP) classification hyperparameters \citep{FilipponeIEEETPAMI13}
on the UCI glass dataset, and (2) the synthetic banana-shaped distribution
of \citet{Haario1999} and a flower-shaped disribution concentrated
on a circle with a periodic perturbation.\vspace{-0.2cm}

\subsection{Pseudo-Marginal MCMC for GP Classification}

\begin{figure*}
\begin{centering}
\includegraphics[width=2.8in]{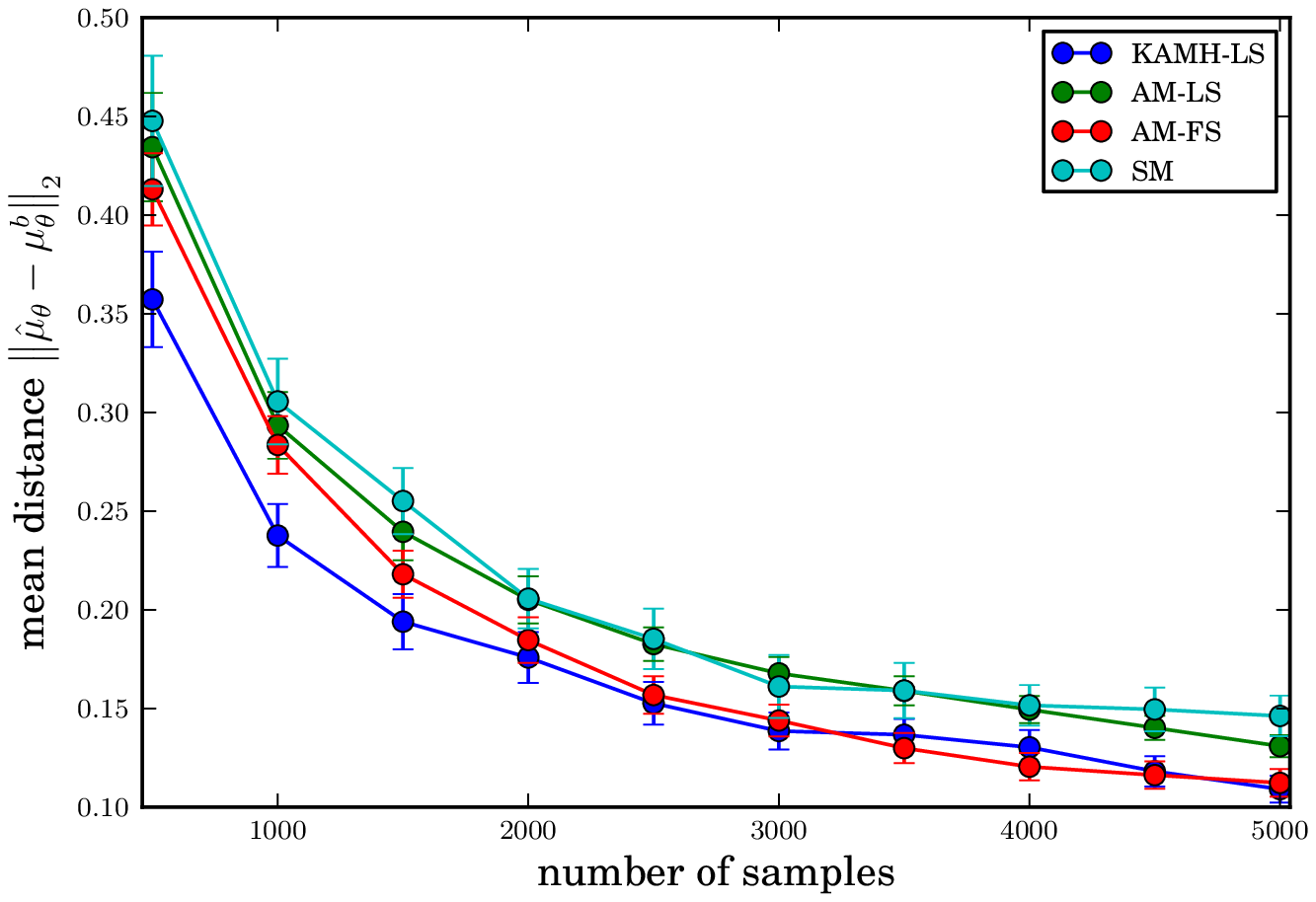}\includegraphics[width=2.8in]{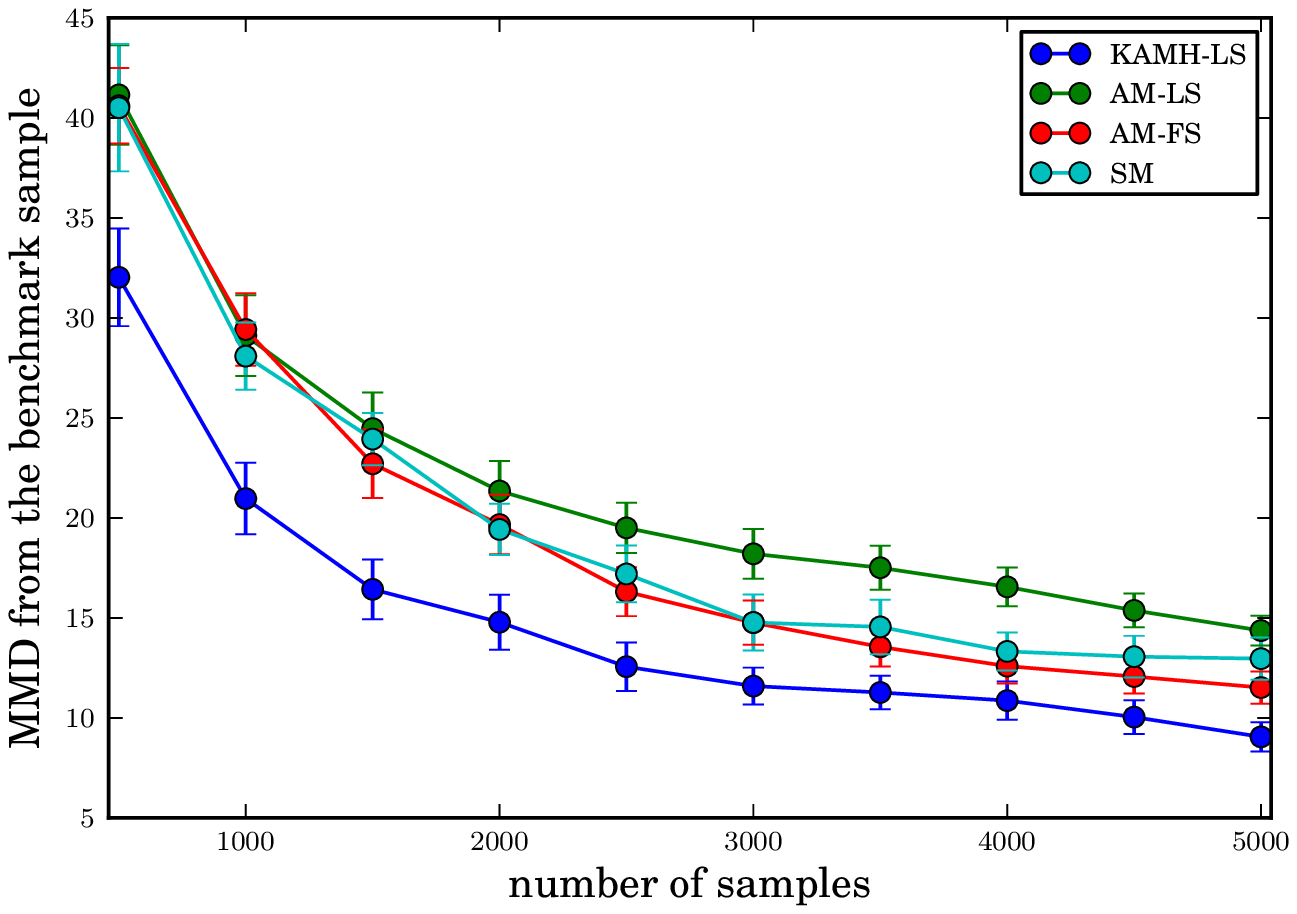}
\par\end{centering}

\caption{\label{fig:glass_comparison}The comparison of \textbf{SM}, \textbf{AM-FS},
\textbf{AM-LS} and \textbf{KAMH-LS} in terms of the distance between
the estimated mean and the mean on the benchmark sample (left) and
in terms of the maximum mean discrepancy to the benchmark sample (right).
The results are averaged over 30 chains for each sampler. Error bars
represent $80\%$-confidence intervals.}
\end{figure*}

\begin{figure*}
\begin{centering}
\textbf{Moderately twisted 8-dimensional} $\mathbf{\mathcal{B}(0.03,100)}$
\textbf{target; iterations: 40000, burn-in: 20000}
\par\end{centering}

\begin{centering}
\includegraphics{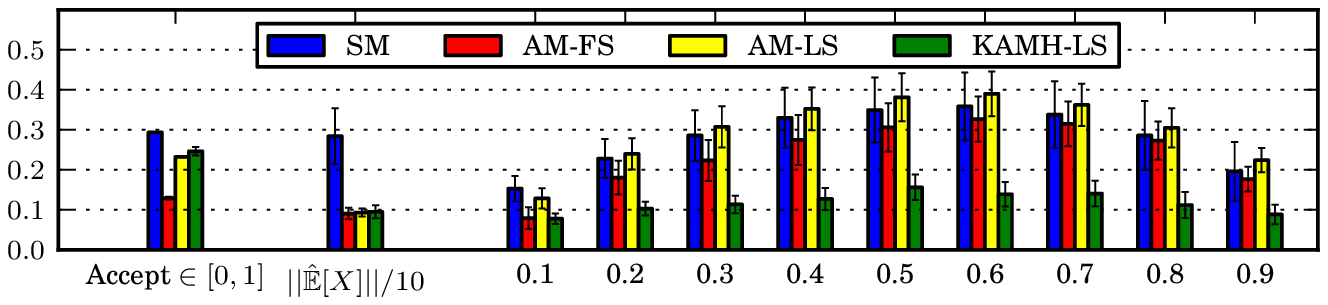}
\par\end{centering}

\begin{centering}
\textbf{Strongly twisted 8-dimensional $\mathbf{\mathcal{B}(0.1,100)}$
target; iterations: 80000, burn-in: 40000}
\par\end{centering}

\begin{centering}
\includegraphics{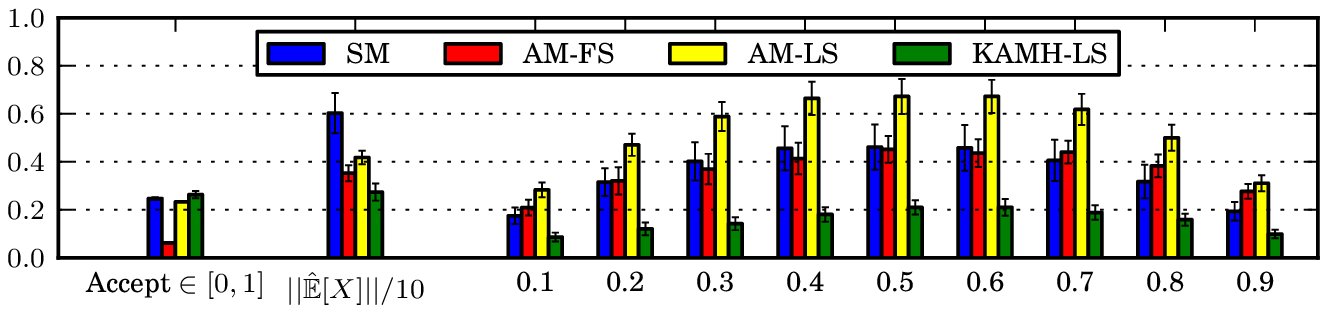}
\par\end{centering}

\begin{centering}
\textbf{8-dimensional }$\mathbf{\mathcal{F}(10,6,6,1)}$\textbf{ target;
iterations: 120000, burn-in: 60000}
\par\end{centering}

\begin{centering}
\includegraphics{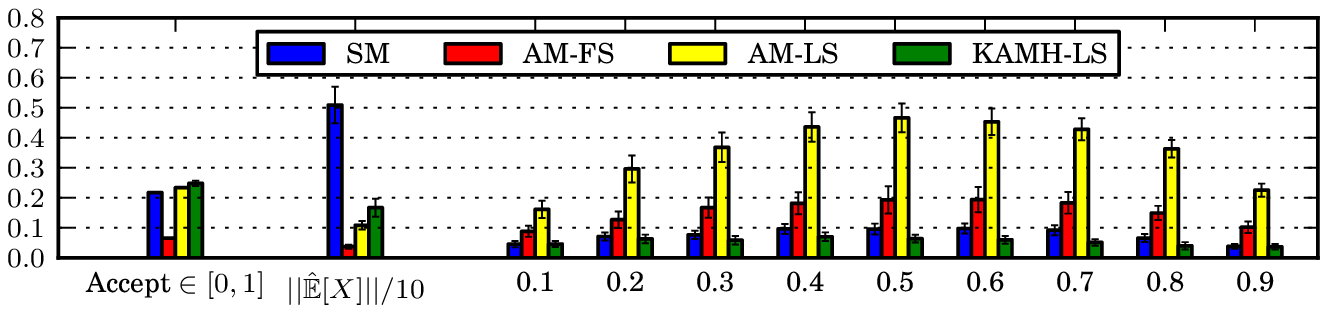}
\par\end{centering}

\centering{}\caption{\label{fig:Results-for-three}Results for three nonlinear targets,
averaged over 20 chains for each sampler. \emph{Accept} is the acceptance
rate scaled to the interval $[0,1]$. The norm of the mean $||\hat{\mathbb{E}}[X]||$
is scaled by 1/10 to fit into the figure scalling, and the bars over
the $0.1,\dots,0.9$-quantiles represent the deviation from the exact
quantiles, scaled by 10; i.e., $0.1$ corresponds to 1\% deviation.
Error bars represent $80\%$-confidence intervals.}
\end{figure*}

In the first experiment, we illustrate usefulness of the MCMC Kameleon
sampler in the context of Bayesian classification with GPs \citep{Williams1998}.
Consider the joint distribution of latent variables $\mathbf{f}$,
labels $\mathbf{y}$ (with covariate matrix \textbf{$X$}), and hyperparameters
$\theta$, given by
\[
p(\mathbf{f},\mathbf{y},\theta)=p(\theta)p(\mathbf{f}|\theta)p(\mathbf{y}|\mathbf{f}),
\]
where $\mathbf{f}|\theta\sim{\cal N}(0,\mathcal{K}_{\theta})$, with
$\mathcal{K}_{\theta}$ modeling the covariance between latent variables
evaluated at the input covariates: $(\mathcal{K}_{\theta})_{ij}=\kappa(\mathbf{x}_{i},\mathbf{x}_{j}'|\theta)=\exp\left(-\frac{1}{2}\sum_{d=1}^{D}\frac{(x_{i,d}-x'_{j,d})^{2}}{\ell_{d}^{2}}\right)$
and $\theta_{d}=\log\ell_{d}^{2}$. We restrict our attention to the
binary logistic classifier; i.e., the likelihood is given by $p(y_{i}|f_{i})=\frac{1}{1-\exp(-y_{i}f_{i})}$
where $y_{i}\in\{-1,1\}$. We pursue a fully Bayesian treatment, and
estimate the posterior of the hyperparameters $\theta$. As observed
by \citet{Murray2012}, a Gibbs sampler on $p(\theta,\mathbf{f}|y)$,
which samples from $p(\mathbf{f}|\theta,y)$ and $p(\theta|\mathbf{f},y)$
in turn, is problematic, as $p(\theta|\mathbf{f},y)$\textbf{ }is
extremely sharp, drastically limiting the amount that any Markov chain
can update $\theta|\mathbf{f},y$. On the other hand, if we directly
consider the marginal posterior $p(\theta|\mathbf{y})\propto p(\mathbf{y}|\theta)p(\theta)$
of the hyperparameters, a much less peaked distribution can be obtained.
However, the marginal likelihood $p(\mathbf{y}|\theta)$ is intractable
for non-Gaussian likelihoods $p(\mathbf{y}|\mathbf{f})$, so it is
not possible to analytically integrate out the latent variables. Recently
developed pseudo-marginal MCMC methods \citep{andrieu2009pseudo}
enable \emph{exact} inference on this problem \citep{FilipponeIEEETPAMI13},
by replacing $p(\mathbf{y}|\theta)$ with an unbiased estimate

\begin{equation}
\hat{p}(\mathbf{y}|\theta):=\frac{1}{n_{\textrm{imp}}}\sum_{i=1}^{n_{\textrm{imp}}}p(\mathbf{y}|\mathbf{f}^{(i)})\frac{p(\mathbf{f}^{(i)}|\theta)}{q(\mathbf{f}^{(i)}|\theta)},
\end{equation}
where $\left\{ \mathbf{f}^{(i)}\right\} _{i=1}^{n_{\textrm{imp}}}\sim q(\mathbf{f}|\theta)$
are $n_{imp}$ importance samples. In \citet{FilipponeIEEETPAMI13},
the importance distribution $q(\mathbf{f}|\theta)$ is chosen as the
Laplacian or as the Expectation Propagation (EP) approximation of
$p(\mathbf{f}|\mathbf{y},\theta)\propto p(\mathbf{y}|\mathbf{f})p(\mathbf{f}|\theta)$,
leading to state-of-the-art results.\vspace{-0.2cm}

We consider the UCI Glass dataset \citep{Bache+Lichman:2013}, where
classification of window against non-window glass is sought. Due to
the heterogeneous structure of each of the classes (i.e., non-window
glass consists of containers, tableware and headlamps), there is no
single consistent set of lengthscales determining the decision boundary,
so one expects the posterior of the covariance bandwidths $\theta_{d}$
to have a complicated (nonlinear) shape. This is illustrated by the
plot of the posterior projections to the dimensions 2 and 7 (out of
9) in Figure \ref{fig:posterior2vs7}. Since the ground truth for
the hyperparameter posterior is not available, we initially ran 30
Standard Metropolis chains for 500,000 iterations (with a 100,000
burn-in), kept every 1000-th sample in each of the chains, and combined
them. The resulting samples were used as a benchmark, to evaluate
the performance of shorter single-chain runs of \textbf{SM}, \textbf{AM-FS},
\textbf{AM-LS} and \textbf{KAMH-LS}. Each of these algorithms was
run for 100,000 iterations (with a 20,000 burnin) and every 20-th
sample was kept. Two metrics were used in evaluating the performance
of the four samplers, relative to the large-scale benchmark. First,
the distance $\left\Vert \hat{\mu}_{\theta}-\mu_{\theta}^{b}\right\Vert _{2}$
was computed between the mean $\hat{\mu}_{\theta}$ estimated from
each of the four sampler outputs, and the mean $\mu_{\theta}^{b}$
on the benchmark sample (Fig. \ref{fig:glass_comparison}, left),
as a function of sample size. Second, the MMD \citep{BorGreRasKrietal06,GreBorRasSchetal07c}
was computed between each sampler output and the benchmark sample,
using the polynomial kernel $\left(1+\left\langle \theta,\theta'\right\rangle \right)^{3}$;
i.e., the comparison was made in terms of all mixed moments of order
up to 3 (Fig. \ref{fig:glass_comparison}, right). The figures indicate
that \textbf{KAMH-LS} approximates the benchmark sample better than
the competing approaches, where the effect is especially pronounced
in the high order moments, indicating that \textbf{KAMH-LS} thoroughly
explores the distribution support in a relatively small number of
samples.\vspace{-0.2cm}

We emphasise that, as for \emph{any} pseudo-marginal MCMC scheme,
neither the likelihood itself, nor any higher-order information about
the marginal posterior target $p(\theta|\mathbf{y})$, are available.
This makes HMC or MALA based approaches such as \citep{Roberts2003,RSSB:RSSB765}
unsuitable for this problem, so it is very difficult to deal with
strongly nonlinear posterior targets. In contrast, as indicated in
this example, the MCMC Kameleon scheme is able to effectively sample
from such nonlinear targets, and outperforms the vanilla Metropolis
methods, which are the \emph{only} competing choices in the pseudo-marginal
context.

In addition, since the bulk of the cost for pseudo-marginal MCMC is
in importance sampling in order to obtain the acceptance ratio, the
additional cost imposed by \textbf{KAMH-LS} is negligible. Indeed,
we observed that there is an increase of only 2-3\% in terms of effective
computation time in comparison to all other samplers, for the chosen
size of the chain history subsample ($n=1000$). \vspace{-0.2cm}

\subsection{Synthetic examples}

\paragraph{Banana target. }

In \citet{Haario1999}, the following family of nonlinear target distributions
is considered. Let $X\sim\mathcal{N}(0,\Sigma)$ be a multivariate
normal in $d\geq2$ dimensions, with $\Sigma=\text{{diag}}(v,1,\ldots,1)$,
which undergoes the transformation $X\to Y$, where $Y_{2}=X_{2}+b(X_{1}^{2}-v)$,
and $Y_{i}=X_{i}$ for $i\neq2$. We will write $Y\sim\mathcal{B}(b,v)$.
It is clear that $\mathbb{E}Y=0$, and that\vspace{-0.5cm} 
\[
\mathcal{B}(y;b,v)=\mathcal{N}(y_{1};0,v)\mathcal{N}(y_{2};b(y_{1}^{2}-v),1)\prod_{j=3}^{d}\mathcal{N}(y_{j};0,1).
\]
\vspace{-0.3cm}

\paragraph{Flower target. }

The second target distribution we consider is the $d$-dimensional
flower target $\mathcal{F}(r_{0},A,\omega,\sigma)$, with\vspace{-0.2cm}
{\small 
\begin{eqnarray*}
 &  & \mathcal{F}(x;r_{0},A,\omega,\sigma)=\\
 &  & \quad\exp\left(-\frac{\sqrt{x_{1}^{2}+x_{2}^{2}}-r_{0}-A\cos\left(\omega\textrm{atan2}\left(x_{2},x_{1}\right)\right)}{2\sigma^{2}}\right)\\
 &  & \qquad\qquad\times\prod_{j=3}^{d}\mathcal{N}(x_{j};0,1).
\end{eqnarray*}
}This distribution concentrates around the $r_{0}$-circle with a
periodic perturbation (with amplitude $A$ and frequency $\omega$)
in the first two dimensions. 

In these examples, exact quantile regions of the targets can be computed
analytically, so we can directly assess performance without the need
to estimate distribution distances on the basis of samples (i.e.,
by estimating MMD to the benchmark sample). We compute the following
measures of performance (similarly as in \citet{Haario1999,Andrieu08})
based on the chain after burn-in: average acceptance rate, norm of
the empirical mean (the true mean is by construction zero for all
targets), and the deviation of the empirical quantiles from the true
quantiles. We consider 8-dimensional target distributions: the moderately
twisted $\mathcal{B}(0.03,100)$ banana target (Figure \ref{fig:Results-for-three},
top) and the strongly twisted $\mathcal{B}(0.1,100)$ banana target
(Figure \ref{fig:Results-for-three}, middle) and $\mathcal{F}(10,6,6,1)$
flower target (Figure \ref{fig:Results-for-three}, bottom).

The results show that MCMC Kameleon is superior to the competing samplers.
Since the covariance of the proposal adapts to the local structure
of the target at the current chain state, as illustrated in Figure
\ref{fig: proposal_contours}, MCMC Kameleon does not suffer from
wrongly scaled proposal distributions. The result is a significantly
improved quantile performance in comparison to all competing samplers,
as well as a comparable or superior norm of the empirical mean. \textbf{SM}
has a significantly larger norm of the empirical mean, due to its
purely random walk behavior (e.g., the chain tends to get stuck in
one part of the space, and is not able to traverse both tails of the
banana target equally well). \textbf{AM} with fixed scale has a low
acceptance rate (indicating that the scaling of the proposal is too
large), and even though the norm of the empirical mean is much closer
to the true value, quantile performance of the chain is poor. Even
if the estimated covariance matrix closely resembles the true global
covariance matrix of the target, using it to construct proposal distributions
at every state of the chain may not be the best choice. For example,
\textbf{AM} correctly captures scalings along individual dimensions
for the flower target (the norm of its empirical mean is close to
its true value of zero) but fails to capture local dependence structure.
The flower target, due to its symmetry, has an isotropic covariance
in the first two dimensions -- even though they are highly dependent.
This leads to a mismatch in the scale of the covariance and the scale
of the target, which concentrates on a thin band in the joint space.
\textbf{AM-LS} has the ``correct'' acceptance rate, but the quantile
performance is even worse, as the scaling now becomes too small to
traverse high-density regions of the target.\vspace{-0.3cm}

\section{\label{sec:Conclusions}Conclusions}

We have constructed a simple, versatile, adaptive, gradient-free MCMC
sampler that constructs a family of proposal distributions based on
the sample history of the chain. These proposal distributions automatically
conform to the local covariance structure of the target distribution
at the current chain state. In experiments, the sampler outperforms
existing approaches on nonlinear target distributions, both by exploring
the entire support of these distributions, and by returning accurate
empirical quantiles, indicating faster mixing. Possible extensions
include incorporating additional parametric information about the
target densities, and exploring the tradeoff between the degree of
sub-sampling of the chain history and convergence of the sampler.

\paragraph{Software.}

Python implementation of MCMC Kameleon is available at \url{https://github.com/karlnapf/kameleon-mcmc}.

\paragraph{Acknowledgments.}

D.S., H.S., M.L.G. and A.G. acknowledge support of the Gatsby Charitable
Foundation. We thank Mark Girolami for insightful discussions and
the anonymous reviewers for useful comments.

\newpage{}

\small \vspace{-0.2cm}\bibliographystyle{icml2014}
\bibliography{bibfile,local}

\normalsize\onecolumn\setcounter{thm}{0}

\appendix

\section{\label{sec:Proofs}Proofs}
\begin{prop}
Let $k$ be a differentiable positive definite kernel. Then \textup{$\nabla_{x}k(x,x)|_{x=y}-2\nabla_{x}k(x,y)|_{x=y}=0$.}\end{prop}
\begin{proof}
Since $k$ is a positive definite kernel there exists a Hilbert space
$\mathcal{H}$ and a feature map $\varphi:\mathcal{\mathbb{R}}^{d}\to\mathcal{H},$
such that $k(x,x')=\left\langle \varphi(x),\varphi(x')\right\rangle _{\mathcal{H}}$.
Consider first the map $\tau:\mathbb{R}^{d}\to\mathbb{R}$, defined
by $\tau(x)=k(x,x)$. We write $\tau=\psi\circ\varphi$, where $\psi:\mathcal{H}\to\mathbb{R}$,
$\psi(f)=\left\Vert f\right\Vert _{\mathcal{H}}^{2}.$ We can obtain
$\nabla_{x}k(x,x)|_{x=y}$ from the Fr\'{e}chet derivative $D\tau(y)\in\mathcal{B}(\mathbb{R}^{d},\mathbb{R})$
of $\tau$ at $y$, which to each $y\in\mathbb{R}^{d}$ associates
a bounded linear operator from $\mathbb{R}^{d}$ to $\mathbb{R}$
\citep[Definition A.5.14]{Steinwart2008book}. By the chain rule for

Fr\'{e}chet derivatives \citep[Lemma A.5.15(b)]{Steinwart2008book},
the value of $D\tau(y)$ at some $x'\in\mathbb{R}^{d}$ is
\begin{eqnarray*}
\left[D\tau(y)\right](x') & = & \left[D\psi\left(\varphi(y)\right)\circ D\varphi(y)\right](x'),
\end{eqnarray*}
where $D\varphi(y)\in\mathcal{B}(\mathbb{R}^{d},\mathcal{H})$, and
$D\psi\left(\varphi(y)\right)\in\mathbb{\mathcal{B}}(\mathcal{H},\mathbb{R})$.
The derivative $D\varphi$ of the feature map exists whenever $k$
is a differentiable function \citep[Section 4.3]{Steinwart2008book}.
It is readily shown that $D\psi\left[\varphi(y)\right]=2\left\langle \varphi(y),\cdot\right\rangle _{\mathcal{H}}$,
so that 
\begin{eqnarray*}
\left[D\tau(y)\right](x') & = & 2\left\langle \varphi(y),\left[D\varphi(y)\right](x')\right\rangle _{\mathcal{H}}.
\end{eqnarray*}
Next, we consider the map $\kappa_{y}(x)=k(x,y)=\left\langle \varphi(x),\varphi(y)\right\rangle _{\mathcal{H}}$,
i.e., $\kappa_{y}=\psi_{y}\circ\varphi$ where $\psi_{y}(f)=\left\langle f,\varphi(y)\right\rangle _{\mathcal{H}}$.
Since $\psi_{y}$ is a linear scalar function on $\mathcal{H}$, $D\psi_{y}\left(f\right)=\left\langle \varphi(y),\cdot\right\rangle _{\mathcal{H}}$.
Again, by the chain rule:
\begin{eqnarray*}
\left[D\kappa_{y}(y)\right](x') & = & \left[D\psi_{y}\left(\varphi(y)\right)\circ D\varphi(y)\right](x')\\
 & = & \left\langle \varphi(y),\left[D\varphi(y)\right](x')\right\rangle _{\mathcal{H}},
\end{eqnarray*}
and thus $\left(D\tau(y)-2D\kappa_{y}(y)\right)(x')=0$, for all $x'\in\mathbb{R}^{d}$,
and we obtain equality of operators. Since Fr\'{e}chet derivatives
can also be written as inner products with the gradients, $\left(\nabla_{x}k(x,x)|_{x=y}-2\nabla_{x}k(x,y)|_{x=y}\right)^{\top}x'=\left(D\tau(y)-2D\kappa_{y}(y)\right)(x')=0$,
$\forall x'\in\mathbb{R}^{d}$, which proves the claim.
\end{proof}

\begin{prop}
\textup{$q_{\mathbf{z}}(\cdot|y)=\mathcal{N}(y,\gamma^{2}I+\nu^{2}\Mz H\Mz^{\top})$.}\end{prop}
\begin{proof}
We start with

\begin{eqnarray*}
p(\beta)p(x^{*}|y,\beta) & = & \frac{1}{\left(2\pi\right)^{\frac{n+d}{2}}\gamma^{d}\nu^{n}}\exp\left(-\frac{1}{2\nu^{2}}\beta^{\top}\beta\right)\\
 &  & \cdot\exp\left(-\frac{1}{2\gamma^{2}}\left(x^{*}-y-\Mz H\beta\right)^{\top}\left(x^{*}-y-\Mz H\beta\right)\right)\\
 & = & \frac{1}{\left(2\pi\right)^{\frac{n+d}{2}}\gamma^{d}\nu^{n}}\exp\left(-\frac{1}{2\gamma^{2}}\left(x^{*}-y\right)^{\top}\left(x^{*}-y\right)\right)\\
 &  & \cdot\exp\left(-\frac{1}{2}\left(\beta^{\top}\left(\frac{1}{\nu^{2}}I+\frac{1}{\gamma^{2}}H\Mz^{\top}\Mz H\right)\beta-\frac{2}{\gamma^{2}}\beta^{\top}H\Mz^{\top}(x^{*}-y)\right)\right).
\end{eqnarray*}
Now, we set
\begin{eqnarray*}
\Sigma^{-1} & = & \frac{1}{\nu^{2}}I+\frac{1}{\gamma^{2}}H\Mz^{\top}\Mz H\\
\mu & = & \frac{1}{\gamma^{2}}\Sigma H\Mz^{\top}(x^{*}-y),
\end{eqnarray*}
and application of the standard Gaussian integral
\begin{eqnarray*}
\int\exp\left(-\frac{1}{2}\left(\beta^{\top}\Sigma^{-1}\beta-2\beta^{\top}\Sigma^{-1}\mu\right)\right)d\beta & =\\
(2\pi)^{n/2}\sqrt{\det\Sigma}\exp\left(\frac{1}{2}\mu^{\top}\Sigma^{-1}\mu\right),
\end{eqnarray*}
leads to
\begin{eqnarray*}
q_{\mathbf{z}}(x^{*}|y) & = & \frac{\sqrt{\det\Sigma}}{\left(2\pi\right)^{\frac{d}{2}}\gamma^{d}\nu^{n}}\exp\left(-\frac{1}{2\gamma^{2}}\left(x^{*}-y\right)^{\top}\left(x^{*}-y\right)\right)\\
 &  & \cdot\exp\left(\frac{1}{2}\mu^{\top}\Sigma^{-1}\mu\right).
\end{eqnarray*}
This is just a $d$-dimensional Gaussian density where both the mean
and covariance will, in general, depend on $y$. Let us consider the
exponent
\begin{eqnarray*}
 &  & -\frac{1}{2\gamma^{2}}\left(x^{*}-y\right)^{\top}\left(x^{*}-y\right)+\frac{1}{2}\mu^{\top}\Sigma^{-1}\mu=\\
 &  & -\frac{1}{2}\Biggl\{\frac{1}{\gamma^{2}}\left(x^{*}-y\right)^{\top}\left(x^{*}-y\right)\\
 &  & \qquad-\frac{1}{\gamma^{4}}\left(x^{*}-y\right)^{\top}\Mz H\Sigma H\Mz^{\top}\left(x^{*}-y\right)\Biggr\}=\\
 &  & -\frac{1}{2}\left\{ \left(x^{*}-y\right)^{\top}R^{-1}\left(x^{*}-y\right)\right\} ,
\end{eqnarray*}
where $R^{-1}=\frac{1}{\gamma^{2}}(I-\frac{1}{\gamma^{2}}\Mz H\Sigma H\Mz^{\top})$.
We can simplify the covariance $R$ using the Woodbury identity to
obtain:
\begin{eqnarray*}
R & = & \gamma^{2}(I-\frac{1}{\gamma^{2}}\Mz H\Sigma H\Mz^{\top})^{-1}\\
 & = & \gamma^{2}\left(I+\Mz H\left(\gamma^{2}\Sigma^{-1}-H\Mz^{\top}\Mz H\right)^{-1}H\Mz^{\top}\right)\\
 & = & \gamma^{2}\left(I+\Mz H\left(\frac{\gamma^{2}}{\nu^{2}}I\right)^{-1}H\Mz^{\top}\right)\\
 & = & \gamma^{2}I+\nu^{2}\Mz H\Mz^{\top}.
\end{eqnarray*}
Therefore, the proposal density is $q_{\mathbf{z}}(\cdot|y)=\mathcal{N}(y,\gamma^{2}I+\nu^{2}\Mz H\Mz^{\top})$. 
\end{proof}

\section{\label{sec:Proposal-Contours-for}Further details on synthetic experiments}

\begin{figure}
\begin{centering}
\includegraphics{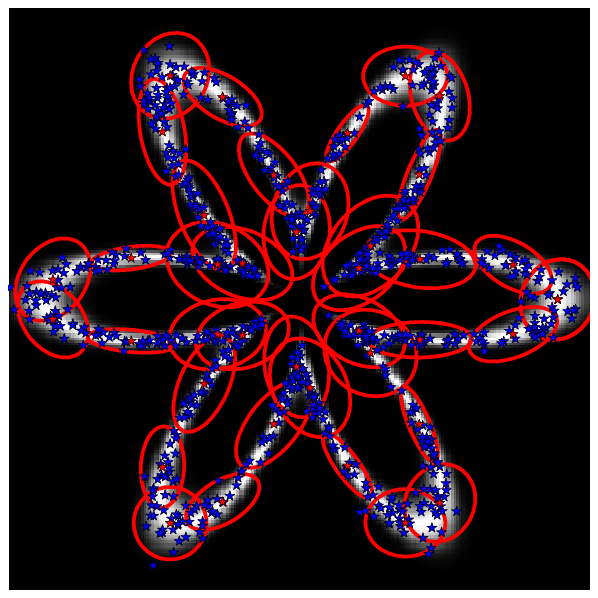}\includegraphics{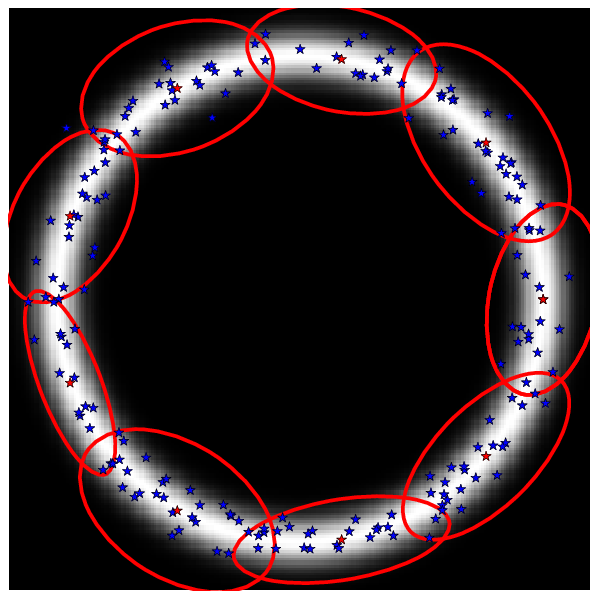}
\par\end{centering}

\caption{\label{fig: proposal_contours-1}$95\%$ contours (red) of proposal
distributions evaluated at a number of points, for the flower and
the ring target. Underneath are the density heatmaps, and the samples
(blue) used to construct the proposals.}
 
\end{figure}

\paragraph{Proposal contours for the Flower target. }

The $d$-dimensional flower target $\mathcal{F}(r_{0},A,\omega,\sigma)$
is given by
\begin{eqnarray*}
\mathcal{F}(x;r_{0},A,\omega,\sigma) & = & \exp\left(-\frac{\sqrt{x_{1}^{2}+x_{2}^{2}}-r_{0}-A\cos\left(\omega\textrm{atan2}\left(x_{2},x_{1}\right)\right)}{2\sigma^{2}}\right)\mathcal{N}(x_{3:d};0,I).
\end{eqnarray*}
This distribution concentrates around the $r_{0}$-circle with a periodic
perturbation (with amplitude $A$ and frequency $\omega$) in the
first two dimensions. For $A=0$, we obtain a band around the $r_{0}$-circle,
which we term the ring target. Figure \ref{fig: proposal_contours-1}
gives the contour plots of the MCMC Kameleon proposal distributions
on two instances of the flower target.

\begin{figure*}
\begin{centering}
\includegraphics[width=2.8in]{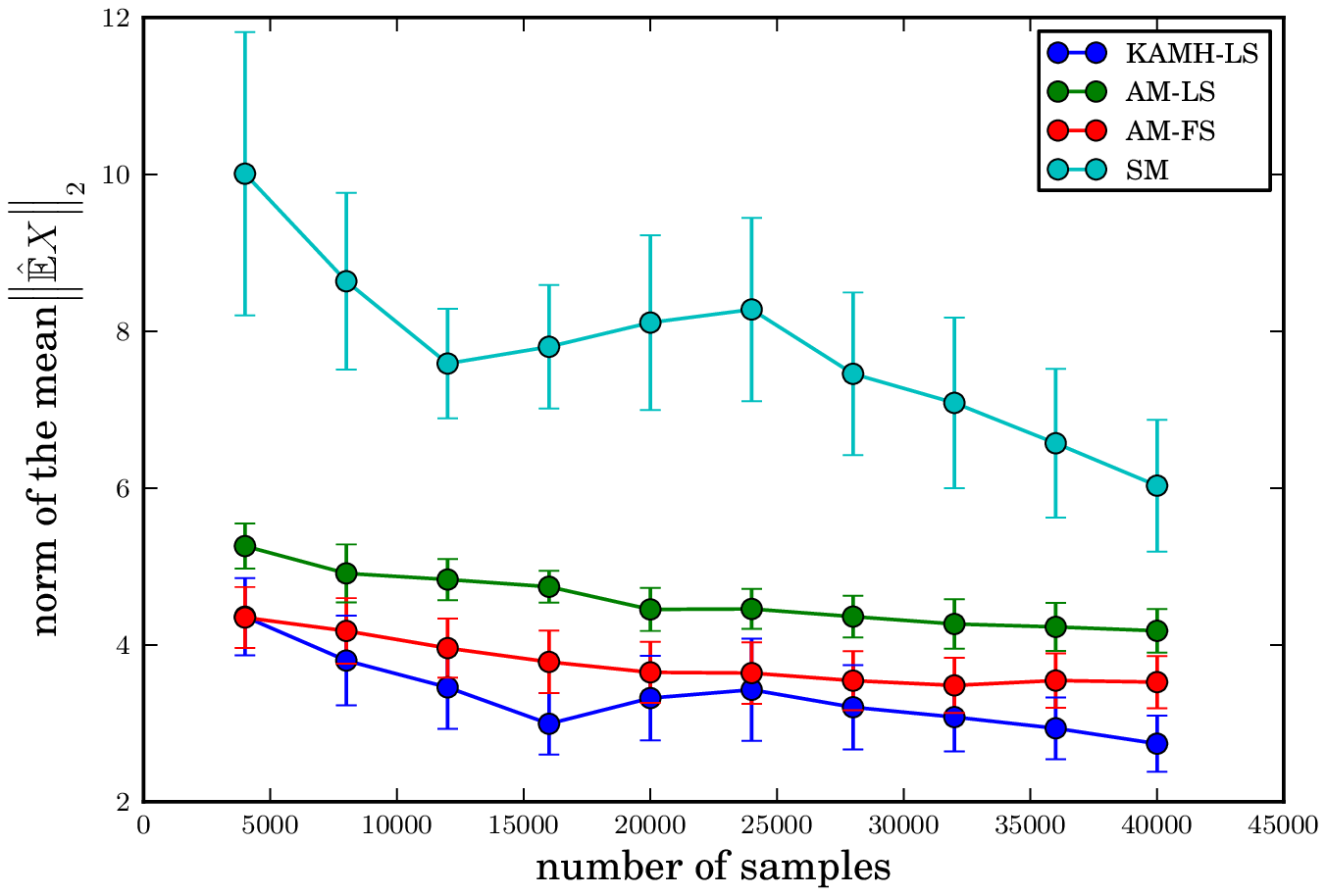}\includegraphics[width=2.8in]{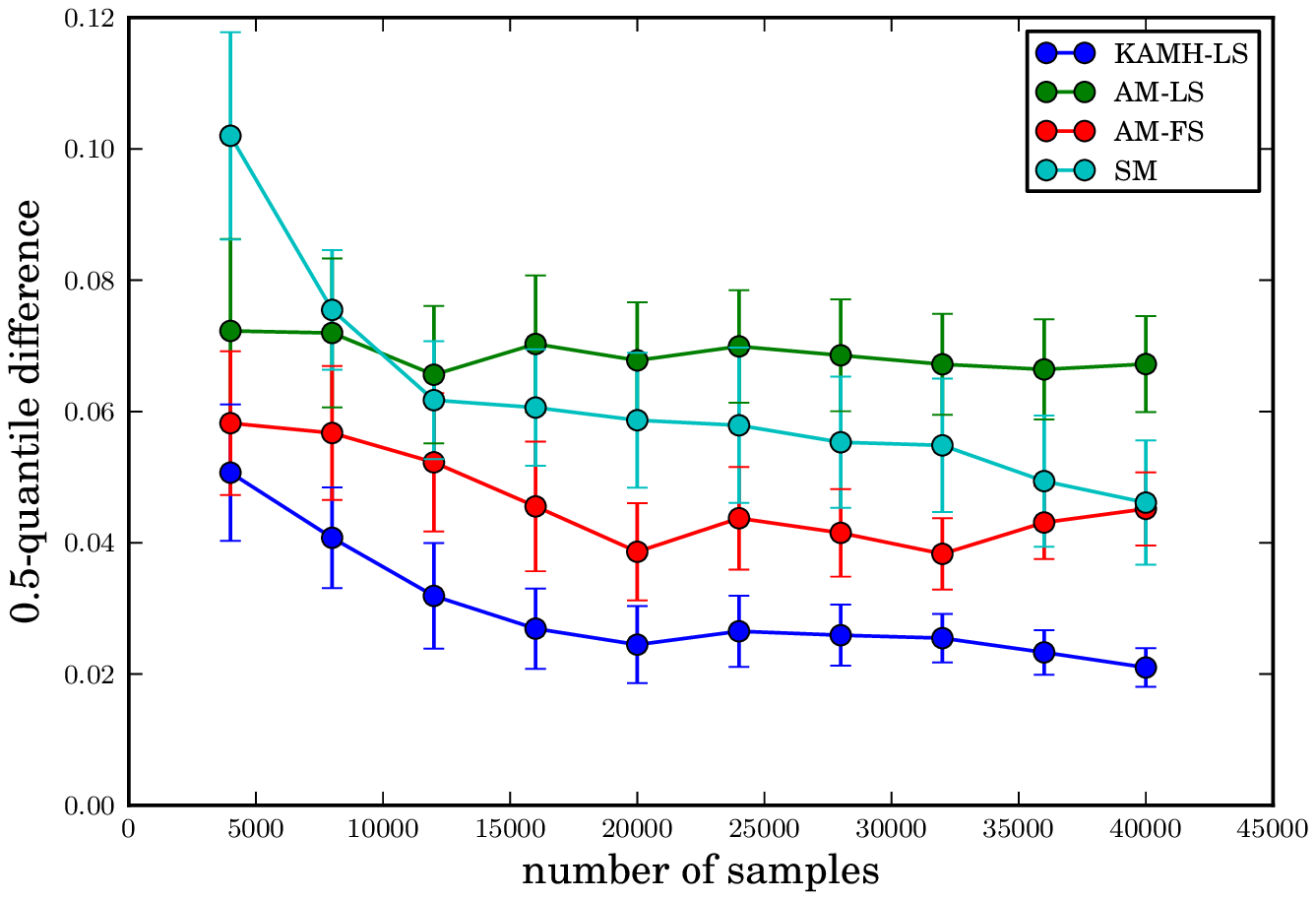}
\par\end{centering}

\caption{\label{fig:banana_drop_comparison}Comparison of \textbf{SM}, \textbf{AM-FS},
\textbf{AM-LS} and \textbf{KAMH-LS} in terms of the norm of the estimated
mean (left) and in terms of the deviation from the 0.5-quantile (right)
on the strongly twisted Banana distribution. The results are averaged
over 20 chains for each sampler. Error bars represent $80\%$-confidence
intervals.}
\end{figure*}

\paragraph{Convergence statistics for the Banana target.}

Figure \ref{fig:banana_drop_comparison} illustrates how the norm
of the mean and quantile deviation (shown for $0.5$-quantile) for
the strongly twisted Banana target decrease as a function of the number
of iterations. This shows that the trends observed in the main text
persist along the evolution of the whole chain.

\section{\label{sec:Principal-components-proposals}Principal Components Proposals}

An alternative approach to the standard adaptive Metropolis, discussed
in \citet[Algorithm 8]{Andrieu08}, is to extract $m\leq d$ principal
eigenvalue-eigenvector pairs $\left\{ \left(\lambda_{j},v_{j}\right)\right\} _{j=1}^{m}$
from the estimated covariance matrix $\Sigma_{\mathbf{z}}$ and use
the proposal that takes form of a mixture of one-dimensional random
walks along the principal eigendirections
\begin{eqnarray}
q_{\mathbf{z}}\left(\cdot|y\right) & = & \sum_{j=1}^{m}\omega_{j}\mathcal{N}(y,\nu_{j}^{2}\lambda_{j}v_{j}v_{j}^{\top}).\label{eq: pca_proposal}
\end{eqnarray}
In other words, given the current chain state $y$, the $j$-th principal
eigendirection is chosen with probability $\omega_{j}$ (choice $\omega_{j}=\lambda_{j}/\sum_{l=1}^{m}\lambda_{l}$
is suggested), and the proposed point is 
\begin{equation}
x^{*}=y+\rho\nu_{j}\sqrt{\lambda_{j}}v_{j},\label{eq: pca_update}
\end{equation}
with $\rho\sim\mathcal{N}(0,1)$. Note that each eigendirection may
have a different scaling factor $\nu_{j}$ in addition to the scaling
with the eigenvalue.

We can consider an analogous version of the update \eqref{eq: pca_update}
performed in the RKHS
\begin{eqnarray}
f & = & k(\cdot,y)+\rho\nu_{j}\sqrt{\lambda_{j}}\mathbf{v}_{j},\label{eq: kpca_update}
\end{eqnarray}
with $m\leq n$ principal eigenvalue-eigenfunction pairs $\left\{ \left(\lambda_{j},\mathbf{v}_{j}\right)\right\} _{j=1}^{m}$.
It is readily shown that the eigenfunctions $\mathbf{v}_{j}=\sum_{i=1}^{n}\tilde{\alpha}_{i}^{(j)}\left[k(\cdot,z_{i})-\muz\right]$
lie in the subspace $\mathcal{H}_{\mathbf{z}}$ induced by $\mathbf{z}$,
and that the coefficients vectors $\tilde{\alpha}^{(j)}=\left(\tilde{\alpha}_{1}^{(j)}\cdots\tilde{\alpha}_{n}^{(j)}\right)^{\top}$
are proportional to the eigenvector of the centered kernel matrix
$HKH$, with normalization chosen so that $\left\Vert \mathbf{v}_{j}\right\Vert _{\mathcal{H}_{k}}^{2}=\left(\tilde{\alpha}^{(j)}\right)^{\top}HKH\tilde{\alpha}^{(j)}=\lambda_{j}\left\Vert \tilde{\alpha}^{(j)}\right\Vert _{2}^{2}=1$
(so that the eigenfunctions have the unit RKHS norm). Therefore, the
update \eqref{eq: kpca_update} has form 
\[
f=k(\cdot,y)+\sum_{i=1}^{n}\beta_{i}^{(j)}\left[k(\cdot,z_{i})-\muz\right],
\]
where $\beta^{\ensuremath{\left(j\right)}}=\rho\nu_{j}\sqrt{\lambda_{j}}\tilde{\alpha}^{(j)}$.
But $\aj=\sqrt{\lambda_{j}}\tilde{\alpha}^{(j)}$ are themselves the
(unit norm) eigenvectors of $HKH,$ as $\left\Vert \aj\right\Vert _{2}^{2}=\lambda_{j}\left\Vert \tilde{\alpha}^{(j)}\right\Vert _{2}^{2}=1$.
Therefore, the appropriate scaling with eigenvalues is already included
in the $\beta$-coefficients, just like in the MCMC Kameleon, where
the $\beta$-coefficients are isotropic. 

Now, we can construct the MCMC PCA-Kameleon by simply substituting
$\beta$-coefficients with $\rho\nu_{j}\alpha^{(j)}$, where $j$
is the selected eigendirection, and $\nu_{j}$ is the scaling factor
associated to the $j$-th eigendirection. We have the following steps:
\begin{enumerate}
\item Perform eigendecomposition of $HKH$ to obtain the $m\leq n$ eigenvectors
$\left\{ \alpha_{j}\right\} _{j=1}^{m}.$
\item Draw $j\sim\text{{Discrete}}\left[\omega_{1},\ldots,\omega_{m}\right]$
\item $\rho\sim\mathcal{N}(0,1)$ 
\item $x^{*}|y,\rho,j\sim\mathcal{N}(y+\rho\nu_{j}M_{\mathbf{z},y}H\aj,\gamma^{2}I)$
($d\times1$ normal in the original space)
\end{enumerate}
Similarly as before, we can simplify the proposal by integrating out
the scale $\rho$ of the moves in the RKHS.
\begin{prop}
\textup{$q_{\mathbf{z}}(\cdot|y)=\sum_{j=1}^{m}\omega_{j}\mathcal{N}(y,\gamma^{2}I+\nu_{j}^{2}\Mz H\alpha^{(j)}\left(\alpha^{(j)}\right)^{\top}H\Mz^{\top})$.}\end{prop}
\begin{proof}
We start with

\begin{eqnarray*}
p(\rho)p(x^{*}\mid y,\rho,j) & \propto & \exp\left[-\frac{1}{2\gamma^{2}}(x^{*}-y)^{\top}(x^{*}-y)\right]\\
 &  & \;\cdot\exp\left[-\frac{1}{2}\left\{ \left(1+\frac{\nu_{j}^{2}}{\gamma^{2}}\ajt H\Mz^{\top}\Mz H\aj\right)\rho^{2}-2\rho\frac{\nu_{j}}{\gamma^{2}}\ajt H\Mz^{\top}\left(x^{*}-y\right)\right\} \right].
\end{eqnarray*}
By substituting

\begin{eqnarray*}
\sigma^{-2} & = & 1+\frac{\nu_{j}^{2}}{\gamma^{2}}\ajt H\Mz^{\top}\Mz H\aj,\\
\mu & = & \sigma^{2}\left(\frac{\nu_{j}}{\gamma^{2}}\ajt H\Mz^{\top}\left(x^{*}-y\right)\right),
\end{eqnarray*}
we integrate out $\rho$ to obtain:

\begin{eqnarray*}
p\left(x^{*}\mid y,j\right) & \propto & \exp\left[-\frac{1}{2}\left\{ \frac{1}{\gamma^{2}}(x^{*}-y)^{\top}(x^{*}-y)-\frac{\nu_{j}^{2}\sigma^{2}}{\gamma^{4}}\left(x^{*}-y\right)^{\top}\Mz H\alpha^{(j)}\left(\alpha^{(j)}\right)^{\top}H\Mz^{\top}\left(x^{*}-y\right)\right\} \right]\\
 & = & \exp\left[-\frac{1}{2}\left(x^{*}-y\right)^{\top}R^{-1}\left(x^{*}-y\right)\right]
\end{eqnarray*}
where $R^{-1}=\frac{1}{\gamma^{2}}\left(I-\frac{\nu_{j}^{2}\sigma^{2}}{\gamma^{2}}\Mz H\alpha^{(j)}\left(\alpha^{(j)}\right)^{\top}H\Mz^{\top}\right)$.
We can simplify the covariance $R$ using the Woodbury identity to
obtain:
\begin{eqnarray*}
R & = & \gamma^{2}(I-\frac{\nu_{j}^{2}\sigma^{2}}{\gamma^{2}}\Mz H\alpha^{(j)}\left(\alpha^{(j)}\right)^{\top}H\Mz^{\top})^{-1}\\
 & = & \gamma^{2}\left(I+\frac{\nu_{j}^{2}\sigma^{2}}{\gamma^{2}}\Mz H\aj\left(1-\frac{\nu_{j}^{2}\sigma^{2}}{\gamma^{2}}\ajt H\Mz^{\top}\Mz H\aj\right)^{-1}\ajt H\Mz^{\top}\right)\\
 & = & \gamma^{2}\left(I+\frac{\nu_{j}^{2}}{\gamma^{2}}\Mz H\alpha^{(j)}\left(\alpha^{(j)}\right)^{\top}H\Mz^{\top}\right)\\
 & = & \gamma^{2}I+\nu_{j}^{2}\Mz H\alpha^{(j)}\left(\alpha^{(j)}\right)^{\top}H\Mz^{\top}.
\end{eqnarray*}
The claim follows after summing over the choice $j$ of the eigendirection
(w.p. $\omega_{j}$).\end{proof}

\end{document}